\documentclass[preprint,3p]{elsarticle}

\usepackage{amsmath,enumerate,subfigure,multirow,hhline,color}
\usepackage{xcolor}
\usepackage{graphicx}
\usepackage{graphics}
\usepackage{amssymb}
\usepackage{amsthm}

\usepackage{epsfig,psfrag}
\usepackage{tabularx}
\usepackage{tabulary}
\usepackage{algorithm}
\usepackage{algpseudocode}

\usepackage[sort]{cite}

\usepackage{hyperref}
\hypersetup{breaklinks=true,colorlinks=true,linkcolor=blue,citecolor=blue}

\usepackage[noabbrev,capitalise,nameinlink]{cleveref}

\DeclareMathOperator{\argmax}{arg\,max}

\newtheorem{theorem}{Theorem}

\newtheorem{assumption}{Assumption}

\newtheorem{corollary}{Corollary}

\newtheorem{lemma}{Lemma}
\newtheorem{definition}{Definition}

\journal{Neurocomputing}

\allowdisplaybreaks

\begin{document}

\begin{frontmatter}

\title{Finite-Time Accuracy of Temporal-Difference Learning Under Schur-Stable Recursions}

\author[1]{Donghwan Lee}\ead{donghwan@kaist.ac.kr}
\author[1]{Dowan Kim}\ead{dwk@sejong.ac.kr}

\address[1]{Department of Electrical Engineering, KAIST, Daejeon, 34141, South Korea}
\address[1]{Department of Defense Systems Engineering, Sejong University, Seoul, 05006,, South Korea}


\begin{abstract}                          
Temporal difference (TD) learning is a cornerstone reinforcement learning (RL) method for policy evaluation, where the goal is to estimate the value function of a Markov decision process under a fixed policy. While a substantial body of work has established its convergence and stability properties, more recent efforts have focused on its statistical efficiency through finite-time error bounds. In this paper, we advance this line of research by developing a new finite-time error analysis for tabular TD learning that directly exploits a discrete-time stochastic linear system representation and leverages Schur stability of the associated matrices. Beyond the specific bounds obtained, the proposed framework provides a reusable template for analyzing TD learning and related RL algorithms, and it offers control-theoretic insights that may guide future developments in finite-sample RL theory.
\end{abstract}

\begin{keyword}                           
Reinforcement learning, TD-learning, convergence analysis, control theory, Lyapunov function
\end{keyword}                             

\end{frontmatter}

\section{Introduction}
Temporal-difference (TD) learning~\citep{sutton1988learning} is one of the most fundamental reinforcement learning (RL) algorithms~\citep{sutton1998reinforcement} for policy evaluation. This idea has served as the basis for a variety of advanced methods, including Q-learning~\citep{watkins1992q}, SARSA~\citep{rummery1994line}, actor-critic algorithms~\citep{konda2000actor}, and more recent developments such as deep Q-learning~\citep{mnih2015human} and gradient TD methods~\citep{sutton2009fast}, among many others. There is a substantial body of work on the theoretical convergence analysis of TD learning~\citep{jaakkola1994convergence,borkar2000ode,tsitsiklis1997analysis}. Traditional approaches primarily focus on asymptotic behavior. Only more recently have researchers investigated guarantees on statistical efficiency and finite-time performance~\citep{bhandari2021finite,dalal2018finite,srikant2019,lakshminarayanan2018linear,hu2019characterizing,li2021sample,zhang2021finite,mou2024optimal,chandak2022concentration,chen2021lyapunov,chen2020finite,wang2021finite,patil2023finite,doan2019finite,sun2020finite}. These advances characterize how fast TD iterates approach the desired solution by providing finite-time error bounds that explicitly depend on the number of time steps. Such analyses shed new light on the statistical properties of TD learning and complement classical asymptotic results.

In this paper, we present a control-theoretic finite-time error analysis of tabular TD learning and derive mean-squared error bounds for both the final iterate and the averaged iterate. Our analysis establishes a connection between tabular TD learning and Schur stability~\citep{chen1995linear} of an associated discrete-time linear system, offering additional insights from a control-theoretic perspective. In contrast, traditional approaches such as the ODE method~\citep{borkar2000ode} and more recent finite-time analyses~\citep{srikant2019,dalal2018finite} typically rely on Hurwitz stability~\citep{chen1995linear} of corresponding continuous-time linear models. Under small or diminishing step sizes, TD learning is approximated by a continuous-time system, which enables the use of Hurwitz stability to establish convergence. However, for tabular TD learning, passing through a continuous-time approximation can introduce redundancy and additional technical overhead.
Our approach bypasses this transition and directly exploits Schur stability of a discrete-time linear system model of TD learning. This leads to simpler proofs, sharper error bounds, and more relaxed step-size conditions. For the final-iterate analysis, our primary focus is on propagating the state correlation matrix, which corresponds to the Lyapunov matrix associated with the system matrix $A^\top$. In contrast, the averaged-iterate analysis leverages standard Lyapunov theory~\citep{khalil2002nonlinear} for the dual system matrix $A$. In this sense, the two analyses exhibit a natural dual relationship.

Finally, we emphasize that this paper focuses on an i.i.d.\ observation model and constant step sizes to simplify the exposition. The i.i.d.\ model is widely used in the literature and is often considered a standard setting (e.g.,~\citep{sutton2009fast,dalal2018finite,borkar2000ode}). The proposed analysis can be extended to more general Markovian observation models using techniques from prior work such as~\citep{srikant2019,bhandari2021finite}. However, this generalization would substantially increase technical complexity and may obscure the core insights of the proposed approach. For clarity, we therefore restrict attention to the i.i.d. observation model.

\section{Related works}
Recently, significant progress has been made in the finite-time error analysis of TD-learning algorithms~\citep{bhandari2021finite,dalal2018finite,srikant2019,lakshminarayanan2018linear,hu2019characterizing,li2021sample,zhang2021finite,mou2024optimal,chandak2022concentration,chen2021lyapunov,chen2020finite,wang2021finite,patil2023finite,doan2019finite,sun2020finite}.
In particular, \citep{dalal2018finite} presented a finite-time error analysis of TD-learning with linear function approximation by exploiting the Hurwitz stability of an associated continuous-time linear system model.
\citep{lakshminarayanan2018linear} studied finite-time error bounds for general linear stochastic approximation algorithms, with TD-learning as a prominent example.
\citep{bhandari2021finite} developed a simple and explicit finite-time analysis for TD-learning, relying on arguments analogous to those used for stochastic gradient descent.
\citep{srikant2019} analyzed TD-learning by controlling both first- and higher-order moments of the error, again leveraging Hurwitz stability of continuous-time linear system models.
\citep{hu2019characterizing} examined TD-learning through the lens of discrete-time Markovian jump linear systems and characterized the resulting algorithmic behaviors.
Using Lyapunov methods for general Markovian stochastic approximation, \citep{chen2021lyapunov,chen2020finite} derived mean-square error bounds and obtained finite-sample mean-squared error guarantees for TD-learning as a corollary.
\citep{mou2024optimal} provided a sharp analysis of averaged linear stochastic approximation schemes and applied it to the averaged iterates of TD-learning.
Sharper sample-complexity bounds for TD-learning were obtained by \citep{li2021sample}.
\citep{wang2021finite} studied biased stochastic approximation under a mild ergodicity-type assumption on the noise sequence and derived finite-time mean-squared error bounds, which were subsequently specialized to TD-learning with linear function approximation, constant step size, and Markovian observation models.
\citep{patil2023finite} established finite-time error bounds for tail-averaged TD-learning under step-size choices that do not require eigenvalue information of the matrix defining the projected TD fixed point.
\citep{chandak2022concentration} derived concentration bounds for stochastic approximation with contractive maps under both martingale-difference and Markov noises, and applied these results to tabular TD-learning.
Finally, \citep{zhang2021finite} provided the first sample-complexity guarantees for average-reward TD-learning with linear function approximation, while \citep{doan2019finite,sun2020finite} developed finite-time error analyses for distributed TD-learning.

Among the works above, we focus on the single-agent Markov decision problem with discounted reward summations~\citep{bhandari2021finite,dalal2018finite,srikant2019,lakshminarayanan2018linear,hu2019characterizing,li2021sample,zhang2021finite,mou2024optimal,chandak2022concentration,chen2021lyapunov,chen2020finite,wang2021finite,patil2023finite}.
The key distinction of the proposed framework is its direct exploitation of Schur stability in a discrete-time linear system model.
By contrast, existing finite-time analyses often rely on Hurwitz stability of continuous-time linear system models~\citep{dalal2018finite,srikant2019,borkar2000ode}, contraction-mapping arguments~\citep{chandak2022concentration}, SGD-inspired techniques~\citep{bhandari2021finite}, or other stochastic-approximation tools~\citep{lakshminarayanan2018linear,li2021sample,wang2021finite,zhang2021finite,mou2024optimal,chen2021lyapunov,chen2020finite,patil2023finite}.
As a result, our analysis offers additional insights through control-theoretic notions, which may be particularly transparent to readers with a control background.
Another advantage is that the constant step-size choice $\alpha\in(0,1)$ does not require prior knowledge of environment parameters; moreover, $\alpha\in(0,1)$ is among the most general step-size conditions adopted in the existing literature. A current limitation is that our analysis is restricted to the tabular setting.
Nevertheless, results tailored to this setting remain valuable, as they provide complementary, and often simpler, insights alongside existing analyses.

\section{Preliminaries}

\subsection{Notation}
The adopted notation is as follows: ${\mathbb R}$: set of real numbers; ${\mathbb R}^n $: $n$-dimensional Euclidean
space; ${\mathbb R}^{n \times m}$: set of all $n \times m$ real
matrices; $A^\top$: transpose of matrix $A$; $A \succ 0$ ($A \prec
0$, $A\succeq 0$, and $A\preceq 0$, respectively): symmetric
positive definite (negative definite, positive semi-definite, and
negative semi-definite, respectively) matrix $A$; $I$: identity matrix with appropriate dimensions; $[A]_{ij}$ is the element of $A$ in $i$-th row and $j$-th column; $\lambda_{\min}(A)$ and $\lambda_{\max}(A)$ for any symmetric matrix $A$: the minimum and maximum eigenvalues of $A$; $|{\mathcal S}|$: cardinality of a finite set $\mathcal S$; ${\rm tr}(A)$: trace of any matrix $A$; $\otimes$: Kronecker product.

\subsection{Markov decision problem}
We consider the infinite-horizon discounted Markov decision problem (MDP)~\citep{bertsekas1996neuro}, where the agent sequentially takes actions to maximize cumulative discounted rewards. In a Markov decision process with the state-space ${\mathcal S}:=\{ 1,2,\ldots ,|{\mathcal S}|\}$ and action-space ${\mathcal A}:= \{1,2,\ldots,|{\mathcal A}|\}$, the decision maker selects an action $a \in {\mathcal A}$ with the current state $s$, then the state
transits to a state $s'$ with probability $P(s'|s,a)$, and the transition incurs a
reward $r(s,a,s')$. For convenience, we consider a deterministic reward function and simply write $r(s_k,a_k ,s_{k + 1}) =:r_{k+1}, k \in \{ 0,1,\ldots \}$. A stochastic policy is a map $\pi:{\mathcal S} \times {\mathcal A}\to [0,1]$ representing the probability, $\pi(a|s)$, of selecting action $a$ at the current state $s$, while a deterministic policy is a map $\pi:{\mathcal S} \to {\mathcal A}$. In this paper, we usually focus on the stochastic policy because the deterministic counterpart can be seen as a special case of the stochastic case. The objective of the Markov decision problem (MDP) is to find a deterministic optimal policy, $\pi^*$, such that the cumulative discounted rewards over infinite-time horizons is
maximized, i.e., $\pi^*:= \argmax_{\pi\in \Theta} {\mathbb E}\left[\left.\sum_{k=0}^\infty {\gamma^k r_{k+1}}\right|\pi\right]$, where $\gamma \in [0,1)$ is the discount factor, $\Theta$ is the set of all admissible deterministic policies, $(s_0,a_0,s_1,a_1,\ldots)$ is a state-action trajectory generated by the Markov chain under policy $\pi$, and ${\mathbb E}[\cdot|\pi]$ is an expectation conditioned on the policy $\pi$. The value function under policy $\pi$ is defined as
\begin{align*}
&V^{\pi}(s)={\mathbb E}\left[ \left. \sum_{k=0}^\infty {\gamma^k r_{k+1}} \right|s_0=s,\pi \right],\quad s\in {\mathcal S}.
\end{align*}

Based on these notions, the policy evaluation problem is defined as follows.
\begin{definition}[Policy evaluation problem]
Given a policy $\pi$, find the corresponding value function $V^\pi$.
\end{definition}
The policy evaluation problem is an important component of policy optimization problems for the Markov decision problem. The (model-free) policy evaluation problem addressed in this paper is defined as follows: given a policy $\pi$, find the corresponding value function $V^\pi$ without the model knowledge, i.e., $P$, only using experiences or transitions $(s,a,r,s')$. The stationary state distribution, if exists, is defined as
\[
\lim_{k \to \infty } {\mathbb P}[s_k= s|\pi] =: d(s), \quad s\in {\mathcal S}.
\]
Throughout, we assume that the induced Markov chain of the underlying MDP with $\pi$ is aperiodic and irreducible so that it has a unique stationary state distribution, and the Markov decision problem is well posed, which is a standard assumption.

\subsection{TD-Learning}
\begin{algorithm}[h!]
\caption{TD-learning}
  \begin{algorithmic}[1]

    \State Initialize $V_0 \in {\mathbb R}^{|{\mathcal S}|}$ arbitrarily such that $\|V_0\|_\infty \le 1$.
    \For{iteration $k=0,1,\ldots$}
        \State Observe $s_k\sim d(\cdot)$, $a_k\sim \pi(\cdot|s_k)$, $s_k'\sim P(\cdot|s_k,a_k)$ and $r_{k+1}= r(s_k,a_k,s_k')$
        \State Update $V_{k+1}(s_k)=V_k(s_k)+\alpha \{r_{k+1}+\gamma V_k(s_k')-V_k (s_k)\}$
    \EndFor
  \end{algorithmic}\label{algo:TD1}
\end{algorithm}
We consider a version of TD-learning given in~\cref{algo:TD1}, where the step-size $\alpha$ is constant in this paper. In this paper, for simplicity of analysis, we assume that the current state $s_k \in {\mathcal S}$ is sampled from the stationary state distribution $d$. Moreover, the transitions $(s_k,a_k,s_k')$ at time step $k$ are i.i.d. samples, where $s_k'$ is the next state sampled at time step $k$, which is different from $s_{k+1}$. Here, $s_{k+1}$ simply means a sample of the current state at time step $k+1$ which is independent of $s_k'$. The i.i.d. observation model is widely employed in the existing literature and serves as the standard setting, e.g.~\citep{sutton2009fast,dalal2018finite,borkar2000ode}. However, the proposed analysis can be extended to incorporate the more general Markovian observation scenarios, utilizing techniques presented in previous works such as~\citep{srikant2019,bhandari2021finite}. Nevertheless, this generalization could significantly increase the complexity of the main analysis and potentially obscure the fundamental insights of the proposed approach. Therefore, to maintain clarity, this work focuses on the i.i.d. observation model.
We make the following assumptions throughout the paper.
\begin{assumption}\label{assumption:summary}$\,$
\begin{enumerate}
\item The step-size $\alpha$ satisfies $\alpha \in (0,1)$.

\item (Positive stationary state distribution) $d(s)> 0$ holds for all $s\in {\mathcal S}$.

\item (Unit bound on rewards) The reward is unit bounded as follows:
\begin{align*}
\max _{(s,a,s') \in {\mathcal S} \times {\mathcal A} \times {\mathcal S}} |r (s,a,s')| \leq 1.
\end{align*}

\item (Unit bound on initial parameter) The initial iterate $V_0$ satisfies $\|V_0\|_\infty \le 1$.

\end{enumerate}
\end{assumption}

The first statement in~\cref{assumption:summary} is crucial for the proposed finite-time error analysis, and this assumption is considered standard in the literature~\citep{gosavi2006boundedness, beck2012error} or even more relaxed compared to existing conditions on the step-sizes~\citep{bhandari2021finite, srikant2019}.
The second statement ensures that every state can be visited infinitely often, facilitating sufficient exploration, which is a standard assumption in the literature~\citep{bertsekas1996neuro}.  The third and forth statements regarding the unit bounds imposed the reward function and $V_0$ are introduced for the sake of simplicity in analysis, without sacrificing generality. To conclude this subsection, we also introduce the concept of boundedness for TD-learning iterates~\citep{gosavi2006boundedness}, which plays a crucial role in our analysis.
\begin{lemma}[Boundedness of TD-learning iterates~\citep{gosavi2006boundedness}]\label{lemma:bounded-Q}
Under~\cref{assumption:summary}, for all $k \ge 0$, we have $\|V_k\|_\infty \le V_{\max}:= \frac{1}{1-\gamma}$.
\end{lemma}

\section{Linear System Model of TD-learning}
In this section, we investigate a discrete-time linear system model that corresponds to~\cref{algo:TD1} and conduct a finite-time error analysis based on stability analysis of linear systems. To enhance clarity and facilitate easy reference, we summarize the definitions of the following notations, which will be extensively utilized throughout this paper.
\begin{definition}\label{def:1}$\,$
\begin{enumerate}
\item Maximum state visit probability: $d_{\max} := \max_{s\in {\mathcal S}} d(s) \in (0,1)$

\item Minimum state visit probability: $d_{\min}:= \min_{s \in {\mathcal S}} d(s) \in (0,1)$

\item Exponential convergence rate: $\rho := 1 - \alpha d_{\min} (1 - \gamma) \in (0,1)$

\item The diagonal matrix $D $ is defined as
\[
D: = \left[ {\begin{array}{*{20}c}
   {d(1)} & {} & {}  \\
   {} &  \ddots  & {}  \\
   {} & {} & {d(|S|)}  \\
\end{array}} \right] \in {\mathbb R}^{|{\mathcal S}|\times |{\mathcal S}|}.
\]

\end{enumerate}
\end{definition}
Bellman equation~\citep{bertsekas1996neuro} can be written in a matrix form as follows:
\begin{align}
V^\pi = \gamma P^\pi V^\pi+ R^\pi,\label{eq:Bellman-eq}
\end{align}
where $P^\pi\in {\mathbb R}^{|{\mathcal S}| \times |{\mathcal S}|}$ is the state transition probability matrix under $\pi$, i.e., $[P^\pi  ]_{ij}  = {\mathbb P}[s' = j|s = i,\pi ]$, and $R^\pi\in {\mathbb R}^{|{\mathcal S}|}$ is the expected reward vector under $\pi$, i.e., $[R^\pi  ]_i  = {\mathbb E}[r(s_k ,a_k ,s_{k + 1} )|a_k  \sim \pi (s_k ),s_k  = i]$. Since $D$ is nonsingular,~\eqref{eq:Bellman-eq} can be equivalently written as
\begin{align}
\gamma D P^\pi V^\pi+ DR^\pi - DV^\pi = 0.\label{eq:Bellman-eq2}
\end{align}
Using the notations introduced, the update in~\cref{algo:TD1} can be equivalently rewritten as
\begin{align}
V_{k+1}=V_k+\alpha \{D R^\pi+\gamma D P^\pi V_k-D V_k + w_{k}\},\label{eq:1}
\end{align}
where
\begin{align}
w_k  :=& e_{s_k } \delta _k  - (DR^\pi   + \gamma DP^\pi  V_k  - DV_k ),\label{eq:w}\\
\delta _k : =& r_{k+1}  + \gamma e_{s_{k + 1} }^\top V_k  - e_{s_k }^\top V_k,\label{eq:TD-error}
\end{align}
and $e_s \in {\mathbb R}^{|{\mathcal S}|}$ is the $s$-th basis vector (all components are $0$ except for the $s$-th component which is $1$). Here, $(s_k,a_k,r_{k+1},s_k')$ is the sample transition in the $k$-th time-step. The expressions can be equivalently reformulated as the linear system
\begin{align}
V_{k + 1} = \underbrace {(I + \gamma \alpha D P^\pi   - \alpha D )}_{ =: A} V_k + \underbrace {\alpha D R^\pi}_{ =: b} + \alpha w_k. \label{eq:2}
\end{align}

Invoking the Bellman equation in~\eqref{eq:Bellman-eq2} leads to the equivalent equation
\begin{align}
\underbrace {V_{k + 1}  - V^\pi  }_{ =: x_{k + 1}} = \underbrace {\{ I + \alpha (\gamma D P^\pi   - D)\} }_{ =: A}\underbrace {(V_k  - V^\pi  )}_{ =: x_k} + \alpha w_k.  \label{eq:TD-learning-stochastic-recursion-form}
\end{align}
Note that the term $b$ in~\eqref{eq:2} has been cancelled out in~\eqref{eq:TD-learning-stochastic-recursion-form} by adding the Bellman equation $-\alpha (\gamma D P^\pi-D )V^\pi- \alpha D R^\pi=0$. Next, defining $x_k := V_k  - V^\pi$ and $A:= I + \alpha (\gamma D P^\pi   - D)$, the TD-learning iteration in~\cref{algo:TD1} can be concisely represented as the \emph{discrete-time stochastic linear system}
\begin{align}
x_{k + 1} = A x_k + \alpha w_k,\quad x_0 \in {\mathbb R}^n,\quad \forall k \geq 0.\label{eq:linear-system-form}
\end{align}
where $n:= |{\mathcal S}|$, and $w_k\in {\mathbb R}^n$ is a stochastic noise. In the remaining parts of this section, we focus on some properties of the above system. The first important property is that the noise $w_k$ has the zero mean, and is bounded. It is formally stated in the following lemma with proofs given in Appendix~\ref{appdx:3}.
\begin{lemma}\label{lemma:bound-W}
Let us define $w_{\max }:=\frac{9}{(1 - \gamma )^2 }$. Then, we have
\begin{enumerate}
\item ${\mathbb E}[w_k] = 0$;

\item ${\mathbb E}[\left\| {w_k } \right\|_\infty  ] \le \sqrt{w_{\max }}$;

\item ${\mathbb E}[\left\| {w_k } \right\|_2 ] \le \sqrt{w_{\max }}$;

\item ${\mathbb E}[w_k^\top w_k ] \le w_{\max }$.
\end{enumerate}
for all $k\geq 0$.
\end{lemma}

To proceed further, let us define the covariance of the noise
\[
{\mathbb E}[w_k w_k^\top ] = :W_k= W_k^\top \succeq 0.
\]
The covariance matrix will play a central role in the proposed analysis.
In particular, an important quantity we use in the main result is the maximum eigenvalue, $\lambda _{\max } (W)$, whose bound can be easily established as follows.
\begin{lemma}\label{lemma:bound-W2}
The maximum eigenvalue of $W$ is bounded as
\[
\lambda _{\max } (W_k) \le w_{\max},\quad \forall k \geq 0,
\]
where $w_{\max} > 0$ is defined in~\cref{lemma:bound-W}.
\end{lemma}
\begin{proof}
The proof is completed by noting $\lambda _{\max } (W_k) \le {\rm tr}(W_k) = {\rm tr}({\mathbb E}[w_k w_k^\top ]) = {\mathbb E}[{\rm tr}(w_k w_k^\top )] = {\mathbb E}[w_k^\top w_k ] \le w_{\max}$, where the inequality comes from~\cref{lemma:bound-W}.
\end{proof}
Lastly, we investigate the property of the system matrix $A$ in~\eqref{eq:linear-system-form}. We establish the fact that the $\infty$-norm of $A$ is strictly less than one, in particular, is bounded by $\rho \in (0,1)$, where $\rho$ is defined in~\cref{def:1}.
\begin{lemma}\label{lemma:max-norm-system-matrix}
$\|A \|_\infty \le \rho$ holds, where the matrix norm  $\| A \|_\infty :=\max_{1\le i \le m} \sum_{j=1}^n {|[A]_{ij} |}$ and $[A]_{ij}$ is the element of $A$ in $i$-th row and $j$-th column.
\end{lemma}
\begin{proof}
Noting that $A = I + \alpha (\gamma DP^\pi   - D)$, we have
\begin{align*}
\sum_j|[A]_{ij}|=&\sum_j {| [I - \alpha D + \alpha \gamma D P^\pi ]_{ij}|}\\
=&  [I-\alpha D ]_{ii} + \sum_j {[\alpha\gamma D P^\pi ]_{ij}}\\
=& 1 - \alpha [D]_{ii} + \alpha \gamma [D]_{ii} \sum_j {[P^\pi]_{ij}}\\
=& 1 - \alpha [D]_{ii} + \alpha \gamma [D]_{ii}\\
=& 1 + \alpha [D]_{ii}(\gamma  - 1),
\end{align*}
where the first line is due to the fact that $A$ is a positive matrix, i.e., all entries are nonnegative. Taking the maximum over $i$, we have
\begin{align*}
\| A \|_\infty = \max_{i\in \{ 1,2,\ldots ,|{\mathcal S}|\} } \{ 1 + \alpha [D]_{ii} (\gamma-1)\}= 1 - \alpha \min_{s \in {\mathcal S}} d(s)(1 - \gamma ) ,
\end{align*}
which completes the proof.
\end{proof}

\cref{lemma:max-norm-system-matrix} plays a crucial role in establishing the finite-time error bounds of the linear system~\eqref{eq:linear-system-form} and, equivalently, in analyzing the finite-time behavior of TD-learning in~\cref{algo:TD1}. Specifically, \cref{lemma:max-norm-system-matrix} states that $A$ is Schur, which can be easily proved by using Gelfand’s formula. Our analysis is based on the insight that when the state of the Schur stable linear system converges to the origin, the corresponding stochastic linear system also approximately converges to the origin on average.

In contrast, traditional approaches, such as the ODE approach~\citep{borkar2000ode} and more recent works~\citep{srikant2019,dalal2018finite}, leverage the Hurwitz property of the continuous-time linear system model. Consequently, by assuming a small or diminishing step-size, which facilitates the transition from continuous-time to discrete-time, TD-learning approximates the continuous-time linear system model. This approximation enables the use of the Hurwitz stability to prove the convergence.

In the context of tabular TD-learning, such conversion steps from the continuous-time domain to the discrete-time domain may introduce additional complexities and redundancy in the analysis. The proposed finite-time error analysis directly employs discrete-time linear system models and the Schur matrix property, and this direct approach yields simpler results, including more straightforward proofs, simpler finite-time error bounds, and a relaxed condition on the step-size $\alpha$. In the next section, we provide the main result, a finite-time error analysis of~\cref{algo:TD1}.

\section{Finite-Time Error Analysis: Final Iteration}
In this section, we investigate the finite-time error bound of the TD-learning in~\cref{algo:TD1} for the final iterate based on the linear system model in~\eqref{eq:linear-system-form}, which is done by analyzing the propagations of both mean and correlation of the state $x_k$.
First of all, taking the mean on~\eqref{eq:TD-learning-stochastic-recursion-form} leads to
\begin{align}
{\mathbb E}[x_{k + 1}] = A{\mathbb E}[x_k],\quad x_0 \in {\mathbb R}^n,\quad \forall k \geq 0, \label{eq:mean-system}
\end{align}
where ${\mathbb E}[w_{k}] = 0$ in~\eqref{eq:TD-learning-stochastic-recursion-form} due to the i.i.d. assumption on the samples of transitions. Therefore, the mean state $\mathbb{E}[x_k]$ follows the behavior of the discrete-time linear system. By utilizing~\cref{lemma:max-norm-system-matrix}, we can establish a finite-time error bound for the mean dynamics driven by~\eqref{eq:mean-system}, which incorporates an exponentially converging term.
\begin{lemma}
For all $k \geq 0$, $\left\| {{\mathbb E}[x_k]} \right\|_\infty$ is bounded as $\left\| {{\mathbb E}[x_k]} \right\|_\infty   \le \rho ^k \left\| x_0 \right\|_\infty, \forall x_0 \in {\mathbb R}^n$.
\end{lemma}
\begin{proof}
Taking the norm on~\eqref{eq:mean-system} leads to $\left\| {{\mathbb E}[x_{k + 1} ]} \right\|_\infty   = \left\| {A {\mathbb E}[x_k ]} \right\|_\infty   \le \left\| A \right\|_\infty  \left\| {\mathbb E} [x_k ] \right\|_\infty   \le \rho \left\| {\mathbb E}[x_k] \right\|_\infty$, where the last inequality is due to~\cref{lemma:max-norm-system-matrix}. Recursively applying the inequality yields the desired conclusion.
\end{proof}

As a next step, we investigate how the covariance matrix , ${\mathbb E}[x_k x_k^\top ]$, propagates over the time. In particular, the covariance matrix is updated through the recursion
\[
{\mathbb E}[x_{k + 1} x_{k + 1}^\top ] = A {\mathbb E}[x_k x_k^\top ]A^\top  + \alpha ^2 W_k,
\]
where ${\mathbb E}[w_k w_k^\top ] = W_k$. Defining $X_k := {\mathbb E}[x_k x_k^\top ], k \geq 0$, it is equivalently written as
\[
X_{k + 1}  = AX_k A^\top  + \alpha^2 W_k,\quad \forall k\geq 0,
\]
with $X_0  := x_0x_0^\top$. It is worth noting that the above recursion corresponds to the Lyapunov matrix associated with the system matrix $A^\top$. In other words, the covariance matrix can be seen as a Lyapunov matrix corresponding to the discrete-time system $x_{k+1} = A^\top x_k$.
A natural question arises regarding the convergence of the iterate, $X_k$, as $k \to \infty$. We can at least prove that $X_k$ is bounded.
\begin{lemma}[Boundedness]\label{lemma:boundedness}
The iterate, $X_k $, is bounded as $\left\| {X_k } \right\|_2  \le \frac{n \alpha^2 w_{\max }}{{1 - \rho ^2 }} + n\left\| {X_0 } \right\|_2$, where $\rho$ is defined in~\cref{def:1}.
\end{lemma}
The proof is given in Appendix~\ref{appdx:1}. Similarly, the following lemma proves that the trace of $X_k$ is bounded, which will be used for the main development.
\begin{lemma}\label{lemma:basic1}
We have the bound ${\rm tr}(X_k)  \le \frac{{36\sigma _{\max }^2 n^2 \alpha }}{{d_{\min } (1 - \gamma )^3 }} + \left\| x_0 \right\|_2^2 n^2 \rho ^{2k}$, where $\rho$ is defined in~\cref{def:1}.
\end{lemma}
The proof is given in Appendix~\ref{appdx:5}.
Now, we provide a finite-time error bound on the mean-squared error motivated by some standard notions in discrete-time linear system theory.
\begin{theorem}\label{thm:main1}
For any $k \geq 0$, we have
\begin{align}
{\mathbb E}[\left\| {V_k  - V^\pi } \right\|_2 ] \le \frac{{6 |{\mathcal S}| \sqrt \alpha  }}{{d_{\min }^{0.5} (1 - \gamma )^{1.5} }} + \left\| V_0 - V^\pi  \right\|_2 |{\mathcal S}|\rho ^k,\label{eq:8}
\end{align}
where $\rho$ is defined in~\cref{def:1}.
\end{theorem}
\begin{proof} Noting the relations
\begin{align*}
{\mathbb E}[\left\| {V_k  - V^\pi } \right\|_2^2 ] =& {\mathbb E}[(V_k  - V^\pi )^\top (V_k  - V^\pi )]\\
 =& {\mathbb E}[{\rm tr}((V_k  - V^\pi )^\top (V_k  - V^\pi ))]\\
 =& {\mathbb E}[{\rm tr}((V_k  - V^\pi )(V_k  - V^\pi )^\top )]\\
 =& {\mathbb E}[{\rm tr}(X_k )],
\end{align*}
and using the bound in~\cref{lemma:basic1}, one gets
\begin{align}
{\mathbb E}[\left\| {V_k  - V^* } \right\|_2^2 ] \le \frac{{36 n^2 \alpha }}{{d_{\min } (1 - \gamma )^3 }} + \left\| {x_0 } \right\|_2^2 n^2 \rho ^{2k}\label{eq:9}
\end{align}
Taking the square root on both side of the last inequality, using the subadditivity of the square root function, the Jensen inequality, and the concavity of the square root function, we have the desired conclusion.
\end{proof}
The first term on the right-hand side of~\eqref{eq:8} can be reduced arbitrarily by decreasing the step size $\alpha \in (0,1)$. The second bound diminishes exponentially as $k \to \infty$, at a rate of $\rho = 1 - \alpha d_{\min} (1 - \gamma) \in (0,1)$.
The bound presented in~\cref{thm:main1} is derived based on the discrete-time linear stochastic system model described in~\eqref{eq:1}. In standard control theory, the stochastic noise term $w_k$ is typically assumed to follow a Gaussian distribution. However, the system in~\eqref{eq:1} incorporates a stochastic noise term $w_k$ with special structures. Specifically, $w_k$ in~\cref{thm:main1} represents discrete random variables that capture essential structures of TD-learning. Furthermore, the noise term is bounded, as demonstrated in~\cref{lemma:bound-W}, which plays a crucial role in establishing the bound in~\cref{lemma:boundedness}. The noise term is also special in the sense that it depends on the state $x_k$, and hence, to establish its boundedness in~\cref{lemma:boundedness}, additional analysis on the boundedness of the state vector $x_k$ is required as shown in~\cref{lemma:bounded-Q}.

Finally, note that while our analysis is primarily conducted in tabular settings, it can be extended to incorporate on-policy linear function approximation with additional effort.

\section{Finite-Time Error Analysis: Averaged Iteration}\label{section:averaged-iterate}
In this section, we additionally provide a finite-time error bound in terms of the averaged iteration $\frac{1}{k}\sum_{i = 0}^{k - 1} {V_i }$. We can first obtain the following result.
\begin{theorem}\label{thm:convergence4}
For any $k\geq0$, it holds that
\begin{align*}
&{\mathbb E}\left[ {\left\| {\frac{1}{k}\sum\limits_{i = 0}^{k - 1} {V_i }  - V^\pi  } \right\|_2 } \right] \le \sqrt {\frac{1}{k}\frac{{|{\mathcal S}|}}{{\alpha d_{\min } (1 - \gamma )}}} \left\| {V_0  - V^\pi  } \right\|_2  + \sqrt {\frac{{36\alpha |{\mathcal S}|^2 }}{d_{\min } (1 - \gamma )^3}}.
\end{align*}
\end{theorem}
The analysis of the final iterate in the previous section is based on the Lyapunov matrix associated with the system matrix $A^\top$. On the other hand, the analysis of the averaged iterate in this section relies on the Lyapunov matrix associated with the dual system matrix $A$. In the analysis of the final iterate, the covariance matrix of the state is propagated. On the other hand, the analysis of the averaged iterate in~\cref{thm:convergence4} draws upon the well-established principles of Lyapunov theory~\citep{chen1995linear,khalil2002nonlinear}, which are summarized below.
\begin{lemma}\label{lmm:Lyapunov-theorem2}
There exists a positive definite $M\succ 0 $ such that $A^\top M A =  M - I$ and $\lambda_{\min}(M) \ge 1,\lambda_{\max}(M) \le \frac{n}{1-\rho}$.
\end{lemma}
The proof of~\cref{lmm:Lyapunov-theorem2} is given in Appendix~\ref{appdx:4}.
Note that in the proof of~\cref{lmm:Lyapunov-theorem2}, we bound $\frac{n}{1 - \rho^2}$ by $\frac{n}{1 - \rho}$ just to simplify the final expression. Note also that $M $ in~\cref{lmm:Lyapunov-theorem2} can be seen as a Lyapunov matrix corresponding to the discrete-time linear system $x_{k+1} = A x_k$, which is dual to the system associated with the covariance matrix in the previous sections.
Now, we are ready to prove~\cref{thm:convergence4}.
\begin{proof}[Proof of~\cref{thm:convergence4}]
Consider the quadratic Lyapunov function, $v(x):= x^\top M x$, where $M$ is a positive definite matrix defined in~\eqref{eq:4}.
Using~\cref{lmm:Lyapunov-theorem2}, we have
\begin{align*}
{\mathbb E}[v(x_{k + 1})]=& {\mathbb E}[(A x_k + \alpha w_k )^\top M(A x_k + \alpha w_k )]\\
\le& {\mathbb E}[v(A x_k)] +  \alpha^2 \lambda _{\max } (M)w_{\max}\\
\le& {\mathbb E}[v(x_k)] - x_k^\top x_k+ \alpha^2 \lambda _{\max } (M)w_{\max},
\end{align*}
where the first inequality follows from~\cref{lmm:Lyapunov-theorem2} and~\cref{lemma:bound-W}. Rearranging the last inequality, summing it over $i=0$ to $k-1$, and dividing both sides by $k$ lead to 
\begin{align*}
\frac{1}{k}\sum\limits_{i = 0}^{k - 1} {{\mathbb E}[x_i^\top x_i]}\le \frac{1}{k}v(x_0) + \alpha^2 \lambda _{\max } (M) w_{\max}.
\end{align*}
Using Jensen's inequality, $\lambda _{\min } (M)\left\| x \right\|_2^2  \le v(x) \le \lambda _{\max } (M)\left\| x \right\|_2^2$, and $\left\| x_0 \right\|_2  \le \sqrt n \left\| {x_0} \right\|_\infty$ with~\cref{lmm:Lyapunov-theorem2}, we have the desired conclusion.
\end{proof}
With a prescribed finial iteration number and the final iteration dependent constant step-size, we can obtain ${\mathcal O}(1/\sqrt T )$ convergence rate with respect to the mean-squared error 
\begin{align*}
{\mathbb E}\left[ {\left\| {\frac{1}{T}\sum\limits_{i = 0}^{T - 1} {V_i }  - V^\pi  } \right\|_2^2 } \right],
\end{align*}
and ${\mathcal O}(1/T^{1/4} )$ convergence rate with respect to 
\begin{align*}
{\mathbb E}\left[ {\left\| {\frac{1}{T}\sum\limits_{i = 0}^{T - 1} {V_i }  - V^\pi  } \right\|_2} \right].
\end{align*}
This result is summarized below.
\begin{corollary}\label{thm:convergence5}
For any final iteration number $T\geq 0$ and the prescribed constant step-size
$\alpha  = \frac{1}{{\sqrt T }}$, we have
\begin{align*}
&{\mathbb E}\left[ {\left\| {\frac{1}{T}\sum\limits_{i = 0}^{T - 1} {V_i }  - V^\pi  } \right\|_2 } \right]\le \frac{1}{{T^{1/4} }}\left( {\sqrt {\frac{|{\mathcal S}|}{{d_{\min } (1 - \gamma )}}} \left\| {V_0  - V^\pi  } \right\|_2  + \sqrt {\frac{{36\alpha \sigma _{\max }^2 |{\mathcal S}|^2 }}{d_{\min } (1 - \gamma )^3}} } \right).
\end{align*}
\end{corollary}
The proof is straightforward, and hence, is omitted here.

\section{Comparative Analysis}
In this section, we compare the proposed result with some recent results in the literature on the finite-time error analysis of TD-learning. \cref{table1} summarizes environmental settings of different but related works in the literature.
\begin{table*}[t]
\caption{Comparative analysis of several results. The symbols $\circ$ and $\times$ indicate whether the corresponding method/analysis is used or not, respectively.
\textbf{Abbrev.:} Func.\ approx.\ = function approximation; Non-i.i.d.\ = non-independent and identically distributed (data); Const.\ step-size = constant step-size; Dim.\ step-size = diminishing step-size; Avg.\ iterate = averaged iterate.}
\centering
\small
\setlength{\tabcolsep}{4pt}
\renewcommand{\arraystretch}{1.15}
\begin{tabular*}{\textwidth}{@{\extracolsep{\fill}}lcccccc}
\hline
Method &
Func.\ approx. &
Non-i.i.d. &
Const.\ step-size &
Dim.\ step-size &
Avg.\ iterate &
Final iterate \\
\hline
Ours & $\times$ & $\times$ & $\circ$ & $\times$ & $\circ$ & $\circ$ \\
\citep{bhandari2021finite} & $\circ$ & $\circ$ (proj.\ req.) & $\circ$ & $\circ$ & $\circ$ & {\color{blue}$\circ$} \\
\citep{lakshminarayanan2018linear} & $\circ$ & $\times$ & $\circ$ & $\times$ & $\circ$ & $\times$ \\
\citep{srikant2019} & $\circ$ & $\circ$ & $\circ$ & $\times$ & $\times$ & $\circ$ \\
\citep{dalal2018finite} & $\circ$ & $\times$ & $\times$ & $\circ$ & $\times$ & $\circ$ \\
\hline
\end{tabular*}
\label{table1}
\end{table*}

\subsection{Comparison with~\citep{bhandari2021finite}}
As presented in~\cref{table1}, the paper by~\citep{bhandari2021finite} explores linear function approximation, both with i.i.d. and non-i.i.d. observation models, constant step-sizes, and considers both averaged and final iterates. Particularly, Theorem~2 in~\citep{bhandari2021finite} establishes a finite-time error bound under the assumption of constant step-size, i.i.d. observation models, and linear function approximation. With the tabular setting, the bound is given by the following inequality for any $\alpha \le \frac{{d_{\min}(1 - \gamma)}}{8}$ and $k\geq 0$:
\begin{align}
{\mathbb E}\left[ {\left\| V_k - V^\pi \right\|}_2 \right] \le \sqrt{\frac{{\exp(-\alpha(1 - \gamma)d_{\min} k)}}{{d_{\min}}}} \left\|{V_0 - V^\pi}\right\|_2 + \sqrt{\frac{{2\alpha \sigma ^2}}{{(1 - \gamma )d_{\min}^2}}},\label{eq:10}
\end{align}
where
\begin{align*}
\sigma^2 = {\mathbb E}\left[{\left\|{(e_{s_k } e_{s_k }^\top )r_{k+1} + \gamma (e_{s_k } e_{s_k '}^\top V^\pi ) - e_{s_k } e_{s_k }^\top V^\pi }\right\|_2^2 }\right].
\end{align*}
Using~\cref{lemma:bounded-Q}, one can readily prove that ${\sigma ^2} \le \frac{9}{{{{(1 - \gamma )}^2}}}$. Using this bound,~\eqref{eq:10} leads to
\begin{align*}
{\mathbb E}\left[ {{{\left\| {{V_k} - {V^\pi }} \right\|}_2}} \right] \le \frac{{3\sqrt {2\alpha } }}{{{d_{\min }}{{(1 - \gamma )}^{1.5}}}} + \sqrt {\frac{{\exp ( - \alpha (1 - \gamma ){d_{\min }}k)}}{{{d_{\min }}}}} {\left\| {{V_0} - {V^\pi }} \right\|_2},
\end{align*}
where the order of the effective horizon $\frac{1}{1-\gamma}$ is identical to those in~\cref{thm:main1} and~\cref{thm:convergence4}.
Compared to~\eqref{eq:10}, the primary advantage of the proposed finite-time error bound in~\eqref{eq:8} lies in that it admits a wider range of constant step-sizes within the interval $(0,1)$, which covers a broader range compared to the one considered in~\citep{bhandari2021finite}. The analysis in~\citep{bhandari2021finite} mirrors techniques for analyzing stochastic gradient descent algorithms, whereas the proposed finite-time error analysis employs a different approach based on the linear system model, thereby providing further insights into TD-learning. Similar to~\citep{bhandari2021finite}, the proposed scheme also utilizes simple and insightful arguments.

\subsection{Comparison with~\citep{lakshminarayanan2018linear}}
In~\citep{lakshminarayanan2018linear}, the authors address the important aspects of linear function approximation, i.i.d. observation models, constant step-sizes, and averaged iterates. They first investigate a general constant step-size averaged linear stochastic approximation and provide a finite-time error bound that depends on the parameters of the general problems. It is worth noting that additional problem-specific analysis is required to derive more explicit finite-time error bounds tailored specifically to TD-learning algorithms. The proof of the finite-time bounds in~\citep{lakshminarayanan2018linear} employs tools used in optimization and stochastic approximation methods, which are known to be conceptually and technically more intricate. In contrast, our proposed approaches utilize tools from standard linear control system theory, employing distinct methodologies that provide additional insights and are more familiar to researchers in the control community.

\subsection{Comparison with~\citep{srikant2019}}

In~\citep{srikant2019}, they investigate finite-time error bounds on TD-learning with linear function approximation considering both i.i.d. and non-i.i.d. observation models, constant step-sizes, and for final iterates. Their work shares conceptual similarities with the proposed approach. The analysis in~\citep{srikant2019} is based on continuous-time linear system models, which are subsequently transformed into discrete-time analysis by assuming a sufficiently small constant step-size $\alpha$. However, in the context of tabular TD-learning, such conversion steps may introduce additional complexities and potentially become redundant. In the tabular setting, the proposed analysis directly employs discrete-time linear system models. This direct approach yields simpler results, including more straightforward finite-time error bounds and a relaxed condition on the step-size $\alpha$.

Moreover, the finite-time error bound presented in Theorem 7 of~\citep{srikant2019} incorporates certain problem-dependent parameters, such as the maximum eigenvalue, $\lambda _{\max } (P)$, and minimum eigenvalue, $\lambda _{\min } (P)$, of the Lyapunov matrix $P$ associated with the Hurwitz system matrix $A = \gamma \Phi D(I - P^\pi  )\Phi ^\top$, i.e., $A^\top P + PA =  - I$. The bound is expressed as follows:
\begin{align*}
{\mathbb E}[\left\| {V_k  - V^\pi  } \right\|_2^2 ] \le \frac{{\lambda _{\max } (P)}}{{\lambda _{\min } (P)}}\left( {1 - \frac{{0.9\alpha }}{{\lambda _{\max } (P)}}} \right)^{k - 1} (0.5\left\| {V_0  - V^\pi  } \right\|_2^2  + 0.5)^2 + \frac{{884\alpha \lambda _{\max } (P)^2 }}{{0.9\lambda _{\min } (P)}},\quad k \geq 1,
\end{align*}
Here, $\Phi\in {\mathbb R}^{|{\mathcal S}| \times m}$ represents the feature matrix, and $\alpha \le \frac{{0.05}}{{125\lambda _{\max } (P)}}$. It is worth noting that these parameters, $\lambda _{\max } (P)$ and $\lambda _{\min} (P)$, exhibit more ambiguity in their dependence on the system parameters compared to the proposed analysis. Consequently, additional problem-specific analysis is necessary in order to fully understand their implications.

\subsection{Comparison with~\citep{dalal2018finite}}

\citep{dalal2018finite} considers linear function approximation, i.i.d. observation models, diminishing step-sizes, and final iterates. In particular, a diminishing step-size of the form $\alpha_k = 1/(k+1)^{-a}$ is addressed, where $a \in (0,1)$ is a problem-dependent parameter, which needs additional problem specific analysis.

\subsection{Comparison with~\citep{lakshminarayanan2018linear}}

The convergence rates of both~\citep{lakshminarayanan2018linear} and our work (averaged iterate) are ${\mathcal O}(1/\sqrt{k})$, while the proposed bounds include a constant bias. This difference arises from the distinct derivation processes employed in each method. The constant bias in our proposed analysis can be attributed to the conservative treatment of randomness in our study. Specifically, in our analysis of the averaged iterate, we encode all the randomness into the stochastic noise term and derive its bound based on a deterministic bound on the noise term. This conservative approach may lead to the constant bias. However, it is important to note that the primary objective of this study is to provide novel control-theoretic perspectives and analysis of TD-learning, rather than focusing solely on achieving faster convergence rates. The techniques presented in this paper aim to foster the synergy between control theory and reinforcement learning theory and offer opportunities for designing and analyzing new reinforcement learning algorithms. Therefore, we consider the new analysis presented in this paper as a complement to, rather than a replacement for, existing analyses in the field.

\section*{Conclusion}
In this paper, we develop a new finite-time error analysis for tabular TD learning in both on-policy and off-policy settings. Our starting point is a direct interpretation of the TD update as a discrete-time linear system driven by stochastic noise. Leveraging Schur stability and standard tools from discrete-time linear systems theory, we obtain a unified analysis with several practical benefits: the proofs become more transparent, the resulting bounds take a simpler form, and the step-size requirement is less restrictive. Importantly, the same Schur-stability viewpoint enables us to establish the convergence of off-policy TD learning within the same framework.

By contrast, much of the related literature proceeds by first analyzing a continuous-time model (e.g., via ODE/Hurwitz arguments) and then transferring conclusions back to the discrete-time algorithm. For tabular TD learning, this detour can be redundant and may introduce unnecessary technical overhead. Our mean-squared error bounds therefore exhibit features that differ from prior results and are stated in a way that accommodates the range of scenarios considered in this paper. Overall, the proposed framework complements existing approaches by offering a control-theoretic, discrete-time perspective on TD learning and related RL algorithms, based on simple concepts and readily applicable analysis tools.

\bibliographystyle{plain}
\bibliography{reference}

\appendix

\section{Proof of~\cref{lemma:bound-W}}\label{appdx:3}
For the first statement, we take the conditional expectation on~\eqref{eq:w} to have ${\mathbb E}[w_k |x_k ] = 0$. Taking the total expectation again with the law of total expectation leads to the first conclusion. Moreover, the conditional expectation, ${\mathbb E}[w_k^\top w_k |V_k ]$, is bounded as
\begin{align*}
{\mathbb E}[w_k^\top w_k |V_k ]=& {\mathbb E}[\left\| {e_{s_k } \delta _k  - (DR^\pi   + \gamma DP^\pi  V_k  - DV_k )} \right\|_2^2 |V_k ]\\
=& {\mathbb E}[\delta _k^2 |V_k ] - {\mathbb E}[\left\| {DR^\pi   + \gamma DP^\pi  V_k  - DV_k } \right\|_2^2|V_k ]\\
\le& {\mathbb E}[\delta _k^2 |V_k ]\\
=& {\mathbb E}[ r_{k+1}^2 |V_k ] + {\mathbb E}[2r_{k+1} \gamma e_{s_{k + 1} }^\top V_k |V_k ]- {\mathbb E}[2r_{k+1} e_{s_k }^\top V_k |V_k ] + {\mathbb E}[\gamma ^2 e_{s_{k + 1} }^\top V_k e_{s_{k + 1} }^\top V_k |V_k ]\\
& - {\mathbb E}[2\gamma e_{s_{k + 1} }^\top V_k e_{s_k }^\top V_k |V_k ]+ {\mathbb E}[e_{s_k }^\top V_k e_{s_k }^\top V_k |V_k ]\\
\le& {\mathbb E}[| r_{k+1}^2 ||V_k ] + {\mathbb E}[|2r_{k+1} \gamma e_{s_{k + 1} }^\top V_k ||V_k ] + {\mathbb E}[|2 r_{k+1} e_{s_k }^\top V_k ||V_k ]\\
& + {\mathbb E}[|\gamma ^2 e_{s_{k + 1} }^\top V_k e_{s_{k + 1} }^\top V_k ||V_k ]] + {\mathbb E}[|2\gamma e_{s_{k + 1} }^\top V_k e_{s_k }^\top V_k ||V_k ] + {\mathbb E}[|e_{s_k }^\top V_k e_{s_k }^\top V_k ||V_k ]\\
\le& \frac{9}{{(1 - \gamma )^2 }} = w_{\max },
\end{align*}
where $\delta_k$ is defined in~\eqref{eq:TD-error}, and the last inequality comes from~\cref{assumption:summary} and~\cref{lemma:bounded-Q}. Taking the total expectation, we have the fourth result. Next, taking the square root on both sides of ${\mathbb E}[\left\| {w_k } \right\|_2^2 ] \le w_{\max}$, one gets ${\mathbb E}[\left\| {w_k } \right\|_\infty  ] \le {\mathbb E}[\left\| {w_k } \right\|_2 ] \le \sqrt {{\mathbb E}[\left\| {w_k } \right\|_2^2 ]}  \le \sqrt {w_{\max}}$, where the first inequality comes from $\left\|  \cdot  \right\|_\infty   \le \left\|  \cdot  \right\|_2$. This completes the proof.

\section{Proof of~\cref{lemma:boundedness}}\label{appdx:1}

We first prove the boundedness of $W_k$ as follows:
\begin{align}
\left\| {W_k } \right\|_2  = \| {{\mathbb E}[w_k w_k^\top ]} \|_2  \le {\mathbb E}[ \| {w_k w_k^\top } \|_2 ] \le w_{\max },
\label{eq:11}
\end{align}
where the first inequality is due to Jensen's inequality, and the last inequality is due to the bound in~\cref{lemma:bound-W}.
Next, noting $X_k  = \alpha ^2 \sum_{i = 0}^{k - 1} {A^i W_{k - i - 1} (A^\top )^i }  + A^k X_0 (A^\top )^k$ and taking the norm on $X_k$ lead to
\begin{align*}
\left\| {X_k } \right\|_2 \le& {\alpha ^2}\sum\limits_{i = 0}^{k - 1} {\left\| W_{k - i - 1} \right\|}_2\left\| A^i \right\|_2^2  + \left\| X_0 \right\|_2 \left\| A^k \right\|_2^2\\
\le& \alpha ^2 \sup _{j \ge 0} \left\| {W_j } \right\|_2 \sum\limits_{i = 0}^{k - 1} {\left\| {A^i } \right\|_2^2 }  + \left\| {X_0 } \right\|_2 \left\| {A^k } \right\|_2^2\\
\le& n \alpha^2 w_{\max }  \sum\limits_{i = 0}^{k - 1} {\left\| {A^i } \right\|_\infty ^2 }  + n\left\| {X_0 } \right\|_2 \left\| {A^k } \right\|_\infty ^2\\
\le& n \alpha^2 w_{\max } \sum\limits_{i = 0}^{k - 1} {\left\| A \right\|_\infty ^{2i} }  + n\left\| {X_0 } \right\|_2 \left\| A \right\|_\infty ^{2k}\\
\le& n \alpha^2 w_{\max } \sum\limits_{i = 0}^{k - 1} {\rho ^{2i} }  + n\left\| {X_0 } \right\|_2 \rho ^{2k}\\
\le& n \alpha^2 w_{\max } \mathop {\lim }\limits_{k \to \infty } \sum\limits_{i = 0}^{k - 1} {\rho ^{2i} }  + n\left\| {X_0 } \right\|_2 \rho ^{2k}\\
\le&  \frac{n \alpha^2 w_{\max }}{{1 - \rho ^2 }} + n\left\| {X_0 } \right\|_2
\end{align*}
where the fourth inequality is due to~\eqref{eq:11}, and the sixth inequality is due to~\cref{lemma:max-norm-system-matrix}, and the last inequality uses $\rho \in (0,1)$. This completes the proof.

\section{Proof of~\cref{lemma:basic1}}\label{appdx:5}
We first bound $\lambda _{\max } (X_k )$ as follows:
\begin{align*}
\lambda _{\max } (X_k ) \le& \alpha ^2 \sum_{i = 0}^{k - 1} {\lambda _{\max } (A^i W_{k - i - 1} (A^\top )^i )}  + \lambda _{\max } (A^k X_0 (A^\top )^k )\\
\le& \alpha ^2 \sup _{j \ge 0} \lambda _{\max } (W_j )\sum_{i = 0}^{k - 1} {\lambda _{\max } (A^i (A^\top )^i )}  + \lambda _{\max } (X_0 )\lambda _{\max } (A^k (A^\top )^k )\\
=& \alpha ^2 \sup _{j \ge 0} \lambda _{\max } (W_j ) \sum_{i = 0}^{k - 1} {\| {A^i }\|_2^2 }  + \lambda _{\max } (X_0 )\| {A^k } \|_2^2\\
\le& \alpha ^2  w_{\max } n\sum_{i = 0}^{k - 1} {\| {A^i } \|_\infty ^2 }  + n\lambda _{\max } (X_0 )\| {A^k } \|_\infty ^2\\
\le& \alpha ^2 w_{\max } n\sum_{i = 0}^{k - 1} {\rho ^{2i} }  + n\lambda _{\max } (X_0 )\rho ^{2k}\\
\le& \alpha ^2 w_{\max } n \lim_{k \to \infty } \sum_{i = 0}^{k - 1} {\rho ^{2i} }  + n\lambda _{\max } (X_0 )\rho ^{2k}\\
\le& \frac{{\alpha ^2 w_{\max } n}}{{1 - \rho ^2 }} + n\lambda _{\max } (X_0 )\rho ^{2k}\\
\le& \frac{{\alpha ^2 w_{\max } n}}{{1 - \rho }} + n\lambda _{\max } (X_0 )\rho ^{2k},
\end{align*}
where the first inequality is due to $A^i W_{k-i-1}(A^\top )^i \succeq 0$ and $A^k X_0 (A^\top )^k \succeq 0$, and the sixth inequality is due to $\rho \in (0,1)$. On the other hand, since $X_k \succeq 0$, the diagonal elements are nonnegative. Therefore, we have ${\rm{tr}}(X_k ) \le n\lambda _{\max } (X_k )$. Combining the last two inequalities lead to
\[
{\rm tr}(X_k ) \le n\lambda _{\max } (X_k ) \le \frac{{\lambda _{\max } (W)n^2 \alpha ^2 }}{{1 - \rho }} + \lambda _{\max } (X_0 )n^2 \rho ^{2k}
\]
Moreover, $\lambda _{\max } (X_0) \le {\rm tr}(X_0 ) = {\rm tr}(x_0 x_0^\top ) = \left\| {x_0 } \right\|_2^2$. Plugging $\rho = 1-\alpha d_{\min}(1-\gamma)$, using the bound on ${\lambda _{\max } (W)}$ in~\cref{lemma:bound-W2}, and the last bound on $\lambda _{\max } (X_0 )$, one gets ${\rm tr} ( X_k) \le \frac{{36\sigma _{\max }^2 n^2 \alpha }}{{d_{\min } (1 - \gamma )^3 }} + \left\| {x_0 } \right\|_2^2 n^2 \rho ^{2k}$. This completes the proof.

\section{Proof of~\cref{lmm:Lyapunov-theorem2}}\label{appdx:4}
Consider matrix $M$ such that
\begin{align}
M = \sum_{k=0}^\infty {(A^k)^\top A^k}. \label{eq:4}
\end{align}
Noting that $A^\top M A + I= A^\top \left(\sum_{k=0}^\infty {(A^k)^\top A^k }\right)A + I = M$, we have $A^\top M A + I = M$, resulting in the desired conclusion. Next, it remains to prove the existence of $M$ by proving its boundedness. In particular, taking the norm on $M$ leads to
\begin{align*}
\left\| M \right\|_2 =& \left\| {I + A^\top A + (A^2 )^\top A^2  +  \cdots } \right\|_2\\
\le& \left\| I \right\|_2  + \left\| {A^\top A} \right\|_2  + \left\| {(A^2 )^\top A^2 } \right\|_2  +  \cdots\\
=& \left\| I \right\|_2  + \left\| A \right\|_2^2  + \left\| {A^2 } \right\|_2^2  +  \cdots\\
=& 1 + n \left\| A \right\|_\infty ^2  + n\left\| {A^2 } \right\|_\infty ^2  +  \cdots\\
=& 1 - n + \frac{n}{1 - \rho^2},
\end{align*}
which implies the boundedness. Next, we prove the bounds on the maximum and minimum eigenvalues.
From the definition~\eqref{eq:4}, $M \succeq I$, and hence $\lambda_{\min}(M)\ge 1$. On the other hand, one gets
\begin{align*}
\lambda_{\max}(M)=& \lambda_{\max}(I + A^\top A+ (A^2)^\top A^2+\cdots)\\
\le& \lambda_{\max}(I) + \lambda_{\max}(A^\top A)+ \lambda_{\max}((A^2)^\top A^2 )+\cdots\\
=& \lambda_{\max}(I) + \|A\|_2^2 + \|A^2\|_2^2  +  \cdots\\
\le& 1 + n \|A\|_\infty^2 + n \|A^2\|_\infty ^2 + \cdots\\
\le& \frac{n}{1 - \rho^2}\\
\le& \frac{n}{1 - \rho}.
\end{align*}
The proof is completed.

\end{document}